\documentclass{article} % For LaTeX2e
\pdfoutput=1

\usepackage{nips12submit_e,times}

\usepackage[numbers]{natbib}
\usepackage{natbibspacing}
\usepackage{epsf}
\usepackage{bbm,amsthm}
\usepackage{amsmath,graphicx}
\usepackage[latin1]{inputenc}
\usepackage{moreverb}
\usepackage{amsmath,graphicx}
\usepackage{moreverb,amssymb,dsfont}
\usepackage{multirow}
\usepackage{graphicx}
\usepackage{calc}
\usepackage{mathtools}
\usepackage{bbm,amsthm}
\usepackage[algo2e,ruled]{algorithm2e}

\usepackage{amsmath}

\newtheorem{theorem}{Theorem}
\newtheorem{lemma}[theorem]{Lemma}

\newcommand{\mb}[1]{\mathbf{#1}}

\newcommand{\md}[1]{\mathds{#1}}

\title{Bellman Error Based Feature Generation\\using Random Projections on Sparse Spaces}

\author{
Mahdi Milani Fard, Yuri Grinberg, Amir-massoud Farahmand, Joelle Pineau, Doina Precup\\
School of Computer Science\\
McGill University\\
Montreal, Canada \\
\texttt{\{mmilan1,ygrinb,amirf,jpineau,dprecup\}@cs.mcgill.ca} \\
}

\nipsfinalcopy % Uncomment for camera-ready version

\begin{document}

\maketitle

\begin{abstract}
We address the problem of automatic generation of features for value function approximation.
Bellman Error Basis Functions (BEBFs) have been shown to improve the error of policy evaluation with function approximation, with a convergence rate similar to that of value iteration. We propose a simple, fast and robust algorithm based on random projections to generate BEBFs for sparse feature spaces. We provide a finite sample analysis of the proposed method, and prove that projections logarithmic in the dimension of the original space are enough to guarantee contraction in the error.  Empirical results demonstrate the strength of this method.
\end{abstract}

\section{Introduction}

The accuracy of parametrized policy evaluation depends on the quality of the features used for estimating the value function. Hence, feature generation/selection in reinforcement learning (RL) has received a lot of attention (e.g. ~\cite{di2010adaptive,kolter2009regularization,keller2006automatic,manoonpong2010extraction,geramifard2011online}). We focus on methods that aim to generate features in the direction of the Bellman error of the current value estimates (Bellman Error Based, or BEBF, features). Successive addition of exact BEBFs has been shown to reduce the error of a linear value estimator at a rate similar to value iteration \cite{parr2007analyzing}. Unlike fitted value iteration~\cite{boyan1995generalization} which works with a fixed feature set, iterative BEBF generation gradually increases the complexity of the hypothesis space by adding new features and thus does not diverge, as long as the error in the generation does not cancel out the contraction effect of the Bellman operator~\cite{parr2007analyzing}.

A number of methods have been introduced in RL to generate features related to the Bellman error, with a fair amount of success \cite{geramifard2011online, di2010adaptive, manoonpong2010extraction, parr2007analyzing, keller2006automatic}, but many of them fail to scale to high dimensional state spaces. In this work, we present an algorithm that uses the idea of applying random projections specifically in very large and sparse feature spaces. In short, we iteratively project the original features into exponentially smaller-dimensional spaces and apply linear regression to temporal differences to approximate BEBFs. We carry out a finite sample analysis that helps determine valid sizes of the projections and the number of iterations. Our analysis holds for both finite and continuous state spaces and is easy to apply with discretized or tile-coded features.

The proposed method is computationally favourable to many other feature extraction methods in high dimensional spaces, in that each iteration takes poly-logarithmic time in the number of dimensions. While providing guarantees on the reduction of the error, it needs minimal domain knowledge, as agnostic random projections are used in the process.

Our empirical analysis shows how the algorithm can be applied to general tile-coded spaces. Our results indicate that the proposed method outperforms both gradient type methods, and also LSTD with random projections~\cite{ghavamzadeh2010lstd}. The algorithm is robust to the choice of parameters and needs minimal tweaking to work. It runs fast and has small memory complexity.

\section{Notations and Background}
Throughout this paper, column vectors are represented by lower case bold letters, and matrices are represented by bold capital letters. $|.|$ denotes the size of a set, and ${\cal M}(\mathcal{X})$ is the set of measures on $\mathcal X$. $\|.\|_0$ is Donoho's zero ``norm'' indicating the number of non-zero elements in a vector. $\|.\|$ denotes the $L^2$ norm for vectors and the operator norm for matrices: $\|\mb M\| = \sup_{\mb v} \|\mb M \mb v\|/\|\mb v\|$. The Frobenius norm of a matrix is the defined as: $\|\mb M\|_F = \sqrt{\sum_{i,j} \mb M^2_{i,j}}$. Also, we denote the Moore-Penrose pseudo-inverse of a matrix $\mb  M$ with $\mb M^{\dagger}$. The weighted $L^2$ norm is defined as:
\begin{eqnarray}
\|f(\mb x)\|_{\rho(\mb x)} = \left(\int \left|f(\mb x)\right|^2 d \rho(\mb x) \right)^{(1/2)}.
\end{eqnarray}

We focus on spaces that are large, bounded and $k$-sparse. Our state is represented by a vector $\mb x \in \mathcal X$ of $D$ features, having $\|\mb x\| \leq 1$. We assume that $\mb x$ is $k$-sparse in some known or unknown basis $\mb \Psi$, implying that $\mathcal{X} \triangleq \{\mb \Psi \mb z, \text{ s.t. } \|\mb z\|_0 \leq k \text{ and } \|\mb z\| \leq 1 \}$. Such spaces occur both naturally (e.g. image, audio and video signals~\cite{olshausen2001learning}) and also from most discretization-based methods (e.g. tile-coding).

\subsection{Markov Decision Process and Fast Mixing}
A \emph{Markov Decision Process} (MDP) {\bf$M = (\mathcal{X}, \mathcal{A}, T, R)$} is defined by a (possibly infinite) set of states $\mathcal{X}$, a set of actions $\mathcal{A}$, a transition probability kernel $T : \mathcal{X} \times \mathcal{A} \rightarrow {\cal M}(\mathcal{X})$, where $T(.|\mb x, a)$ defines the distribution of next state given that action $a$ is taken in state $\mb x$, and a (possibly stochastic) reward function $R: \mathcal{X} \times \mathcal{A} \rightarrow {\cal M}([0, R_{\max}])$.
Throughout the paper, we focus on discounted-reward MDPs, with the discount factor denoted by $\gamma \in [0,1)$. At discrete time steps, the reinforcement learning agent chooses an action and receives a reward. The environment then changes to a new state according to the transition kernel.

A \emph{policy} is a (possibly stochastic) function from states to actions. The \emph{value of a state} $\mb x$ for policy $\pi$, denoted by $V^\pi(\mb x)$, is the expected value of the discounted sum of rewards ($\sum_t \gamma^t r_t$) if the agent starts in state $\mb x$ and acts according to policy $\pi$. Defining $R(\mb x,\pi(\mb x))$ to be the expected reward at point $\mb x$ under policy $\pi$, the value function satisfies the Bellman equation: 
\begin{equation}
V^\pi(\mb x) = R(\mb x,\pi(\mb x)) + \gamma \int V^\pi(\mb y) T(d \mb y|\mb x,\pi(\mb x)).
\end{equation}

There are many methods developed to find the value of a policy (policy evaluation) when the transition and reward functions are known. Among these there are dynamic programming methods in which one iteratively applies the \emph{Bellman operator}~\cite{sutton98} to an initial guess of the optimal value function. The Bellman operator $\mathcal T$ on a value estimate $V$ is defined as:
\begin{equation}
\mathcal T V(\mb x) = R(\mb x,\pi(\mb x)) + \gamma \int V(\mb y) T(d \mb y|\mb x,\pi(\mb x)),
\end{equation}
When the transition and reward models are not known, one can use a finite sample set of transitions to learn an approximate value function. Least-squares temporal difference learning (LSTD) and its derivations~\cite{boyan2002,lagoudakis2003least} are among the methods used to learn a value function based on a finite sample. LSTD type methods are efficient in their use of data, but fail to scale to high dimensional state spaces due to extensive computational complexity. Using LSTD in spaces induced by random projections is a way of dealing with such domains~\cite{ghavamzadeh2010lstd}. Stochastic gradient descent type method are also used for value function approximation in high dimensional state spaces, some with proofs of convergence in online and offline settings~\cite{maei2010gq}. However gradient type methods typically have slow convergence rates and do not make efficient use of the data.

To arrive at a finite sample bound on the error of our algorithm, we assume certain mixing conditions on the Markov chain in question. We assume that the Markov chain \emph{uniformly quickly forgets its past} (defined in detail in the appendix). There are many classes of chains that fall into this category (see e.g. \cite{farahmand2011model}). Conditions under which a  Markov chain uniformly quickly forgets its past are of major interest and are discussed in the appendix.

\subsection{Bellman Error Based Feature Generation}

In high-dimensional state spaces, direct estimation of the value function fails to provide good results with small numbers of sampled transitions. Feature selection/extraction methods have thus been used to build better approximation spaces for the value functions~\cite{di2010adaptive,kolter2009regularization,keller2006automatic,manoonpong2010extraction,geramifard2011online}. Among these, we focus on methods that aim to generate features in the direction of the the \emph{Bellman error} defined as:
\begin{eqnarray}
e_V(.) = \mathcal TV(.) - V(.).
\end{eqnarray}
Let $S_n = ( (\mb x_t,r_t)_{t=1}^n )$ be a random sample of size $n$, collected on an MDP with a fixed policy. Given an estimate $V$ of the value function, \emph{temporal difference (TD) errors} are defined to be:
\begin{eqnarray}
\delta_t = r_t + \gamma V(\mb x_{t+1}) - V(\mb x_{t}).
\end{eqnarray}
It is easy to show that the expectation of the temporal difference given a point $\mb x_t$ equals the Bellman error on that point~\cite{sutton98}. TD-errors are thus proxies to estimating the Bellman error.

Using temporal differences, \citet{menache2005basis} introduced two algorithms to adapt basis functions as features for linear function approximation. \citet{keller2006automatic} applied neighbourhood component analysis as a dimensionality reduction technique to construct a low dimensional state space based on the TD-error. In their work, they iteratively add feature that would help predict the Bellman error. \citet{parr2007analyzing} later showed that any BEBF extraction method with small angular approximation error will provably tighten approximation error in the value function estimate.

Online feature extraction methods have also been studied in the RL literature. \citet{geramifard2011online} have recently introduced the \emph{incremental Feature Dependency Discovery} (iFDD) as a fast online algorithm to extract non-linear binary feature for linear function approximation. In their work, one keeps a list of candidate features (non-linear combination of two active features), and among these adds the features that correlates the most with the TD-error.

In this work, we propose a method that generates BEBFs using linear regression in a small space induced by random projection. We first project the state features into a much smaller space and then regress a hyperplane to the TD-errors. For simplicity, we assume that regardless of the current estimate of the value function, the Bellman error is always linearly representable in the original feature space. This seems like a strong assumption, but is true, for example, in virtually any discretized space, and is also likely to hold in very high dimensional feature spaces\footnote{In more general cases, the analysis has to be done with respect to the \emph{projected} Bellman error (see e.g. \cite{parr2007analyzing}). We assume linearity of the Bellman error to simplify the derivations.}.

\subsection{Random Projections and Inner Product}
It is well known that random projections of appropriate sizes preserve enough information for exact reconstruction with high probability (see e.g. \cite{davenport2006detection, candes2008introduction}). This is because random projections are norm and distance-preserving in many classes of feature spaces~\cite{candes2008introduction}. 

There are several types of random projection matrices that can be used. In this work, we assume that each entry in a projection $\mb  \Phi^{D \times d}$ is an i.i.d. sample from a Gaussian~\footnote{The elements of the projection are typically taken to be distributed with $\mathcal{N} (0, 1/D)$, but we scale them by $\sqrt{D/d}$, so that we avoid scaling the projected values (see e.g. \cite{davenport2006detection}).}:
\begin{eqnarray}
\label{label:rndproj}
\phi_{i,j} = \mathcal{N} (0, 1/d).
\end{eqnarray}

Recently, it has been shown that random projections of appropriate sizes preserve linearity of a target function on sparse feature spaces. A bound introduced in \cite{fard2012comp} and later tightened in \cite{fard2012olstech} shows that if a function is linear in a sparse space, it is almost linear in an exponentially smaller projected space. An immediate lemma based on Theorem 2 of \cite{fard2012olstech} bounds the bias induced by random projections:

\begin{lemma} \label{lemma:linsparse}
Let $\mb \Phi^{D \times d}$ be a random projection according to Eqn~\ref{label:rndproj}. Let $\mathcal X$ be a $D$-dimensional $k$-sparse space. Fix $\mb w \in \md R^{D}$ and $1>\xi>0$. Then, with probability $>1-\xi$:
\begin{eqnarray}
\forall \mb x \in \mathcal X: \left| \langle \mb \Phi^T \mb w, \mb \Phi^T \mb x \rangle - \langle \mb w, \mb x \rangle \right| \leq \epsilon^{(\xi)}_{\text{prj}} \|\mb w\| \|\mb x\|,
\end{eqnarray}
where $ \epsilon^{(\xi)}_{\text{prj}} = \sqrt{\frac{48 k }{d}\log \frac{4 D}{\xi}}$.
\end{lemma}

Hence, projections of size $\tilde O(k \log D)$ preserve the linearity up to an arbitrary constant. Along with the analysis of the variance of the estimators, this helps us bound the prediction error of the linear fit in the compressed space.

\section{Compressed Linear BEBFs}

%In this section, we focus on finite state MDPs and will use matrix notation. Let ${\mb v}$ be the vector of values assigned to the states. Let $\mb T^\pi$ be the transition matrix under policy $\pi$, and $\mb r^\pi$ be the expected reward vector under that policy. Overloading the notation, let $\mathcal T$ be the Bellman operator: $\mathcal T {\mb v} \eqdef {\mb r^\pi} + \gamma {\mb P^\pi \mb v}$. The Bellman error is defined as the difference between the value function and the result of the Bellman operator applied on the value: $BE(\mb v) \eqdef \mathcal T \mb v - \mb v$, and $BE_{\mb x}(\mb v)$ is the Bellman error at point $\mb x$.

Linear function approximators can be used to estimate the value of a given state. Let $V_m$ be an estimated value function described in a linear space defined by a feature set $\{\psi_1, \dots \psi_m\}$. \citet{parr2007analyzing} show that if we add a new BEBF $\psi_{m+1} = e_{V_m}$ to the feature set, (with mild assumptions) the approximation error on the new linear space shrinks by a factor of $\gamma$. They also show that if we can estimate the Bellman error within a constant angular error, $\cos^{-1}(\gamma)$, the error will still shrink.

Estimating the Bellman error by regressing to temporal differences in high-dimensional sparse spaces can result in large prediction error. However, as discussed in Lemma~\ref{lemma:linsparse}, random projections were shown to exponentially reduce the dimension of a sparse feature space, only at the cost of a controlled constant bias. A variance analysis along with proper mixing conditions can also bound the estimation error due to the variance in MDP returns. One can thus bound the total prediction error with much smaller number of sampled transitions when the regression is applied in the compressed space.

In light of these results, we propose the \emph{Compressed Bellman Error Based Feature Generation} algorithm (CBEBF). To simplify the bias--variance analysis and avoid multiple levels of regression, we present here a simplified version of compressed BEBF-based regression, in that new features are added to the value function approximator with constant weight 1 (i.e. no regression is applied on the generated BEBFs):

\begin{algorithm2e}
\caption{Simplified Compressed BEBFs}
\label{alg:net}
\KwIn{Sample trajectory $S_n = ( (\mb x_t,r_t)_{t=1}^n )$, where $\mb x_t$ is the observation received at time $t$, and $r_t$ is the observed reward; Number of BEBFs: $m$; Projection size schedule: $d_1, d_2, \dots, d_m$}
\KwOut{$\mb w$: the linear coefficient of the value function approximator}
$\mb w^{D \times 1} \leftarrow 0$\;
\For{$i\leftarrow 1$ \KwTo $m$}{
  Generate random projection $\mb \Phi^{D \times d_i}$ according to Eqn~\ref{label:rndproj}\;
  Calculate TD-errors: $\delta_t = r_t + \gamma \mb x^T_{t+1} \mb w - \mb x^T_{t} \mb w$\;
  Let $\mb w'^{d_i \times 1}$ be the ordinary least-squares parameter using $\mb \Phi^T \mb x_t$ as inputs and $\delta_t$ as outputs\;
  Update $\mb w \leftarrow \mb w + \mb \Phi \mb w'$\;
}
\end{algorithm2e}

The optimal number of BEBFs and the schedule of projection sizes need to be determined and are subjects of future discussion. But we show in the next section that logarithmic size projections should be enough to guarantee the reduction of error in value function prediction at each step. This makes the algorithm very attractive when it comes to computational and memory complexity, as the regression at each step is only on a small projected feature space. As we discuss in our empirical analysis, the algorithm is very fast and robust with respect to the selection of parameters.

One can view the above algorithm as a model selection procedure that gradually increases the complexity of the hypothesis space by adding more BEBFs to the feature set. This means that the procedure has to be stopped at some point to avoid over-fitting. This is relatively easy to do, as one can use a validation set and compare the estimated values against the empirical returns. The generation of BEBFs should stop when the validation error starts to rise.

\subsection*{Finite Sample Analysis}
This section provides a finite sample analysis of the proposed algorithm. Parts of the analysis are not tight and could use further work, but the bound suffices to prove reduction of the error as new BEBFs are added to the feature set.

The following theorem shows how well we can estimate the Bellman error by regression to the TD-errors in a compressed space. It highlights the bias--variance trade-off with respect to the choice of the projection size.

\begin{theorem} \label{theorem:cbebf}
Let $\mb \Phi^{D \times d}$ be a random projection according to Eqn~\ref{label:rndproj}. Let $S_n = ( (\mb x_t,r_t)_{t=1}^n)$ be a sample trajectory collected on an MDP with a fixed policy with stationary distribution $\rho$, in a $D$-dimensional $k$-sparse feature space. Fix any estimate $V$ of the value function, and the corresponding TD-errors $\delta_t$'s bounded by $\pm \delta_{\max}$. Assume that the Bellman error is linear in the features with parameter $\mb w$. For OLS regression we have $\mb w^{(\Phi)}_{\text{ols}} = (\mb X \mb \Phi)^\dagger \mb \delta$, where $\mb X$ is the matrix containing $\mb x_t$'s and $\mb \delta$ is the vector of TD-errors. Assume that $\mb X$ is of rank larger than $d$. There exist constants $c_{1 \dots 4}$ depending only on the mixing conditions of the chain, such that for any fixed $0<\xi_{1\dots 5}<1$, with probability no less than $1-(\xi_1+\xi_2+\xi_3+\xi_4+\xi_5)$:
\begin{eqnarray}
\hspace{-10px} \left\| \mb x^T \mb \Phi \mb w^{(\Phi)}_{\text{ols}} - e_V(\mb x) \right\|_{\rho(\mb x)} \hspace{-10px}
&\leq& \epsilon^{(\xi_1)}_{\text{prj}} \left\| \mb w \right\| \left( 3 + m_{\max} \left\| (\mb X \mb \Phi)^\dagger\right\| \left( \sqrt \frac{n}{d} + \sqrt[4]{ c_3 n d \, \log\frac{c_4 \sqrt d}{\xi} } \right) \right) \\
&& + \frac{\delta_{\max} m_{\max}}{n} \left\|\mb  \Sigma^{-1}_\Phi\right\|   \| \mb X \mb \Phi \| \sqrt {2 k \log \frac{2D}{\xi_3}} \\
&& + \delta_{\max} m^3_{\max} \sqrt{\frac{d^3}{n^3}} \left\|\mb \Sigma^{-1}_\Phi\right\|^2   \| \mb X \mb \Phi \| \sqrt {2 c_3 \log \frac{c_4 d^2}{\xi_4} \log \frac{2 d}{\xi_5}}\\
&& + \tilde O(n^{-2}),
\end{eqnarray}
where $\epsilon^{(\xi_1)}_{\text{prj}}$ is according to Lemma~\ref{lemma:linsparse}, $m_{\max} = \max_{\mb z \in \mathcal X} \left\|\mb z^T \mb \Phi \right\|$ and $\mb \Sigma_\Phi$ is the feature covariance matrix under measure $\rho$.
\end{theorem}

Detailed proof is included in the appendix. The sketch of the proof is as follows: Lemma~\ref{lemma:linsparse} suggests that if the Bellman error is linear in the original features, the bias due to the projection can be bounded within a controlled constant error with logarithmic size projections (first line in the bound). If the Markov chain ``forgets'' exponentially fast, one can bound the \emph{on-measure} variance part of the error by a constant error with similar sizes of sampled transitions~\cite{samson2000concentration} (second and third line of the bound).

Theorem~\ref{theorem:cbebf} can be further simplified by using concentration bounds on random projections as defined in Eqn~\ref{label:rndproj}. The norm of $\mb \Phi$ can be bounded using the bounds discussed in~\cite{candes2006near}; we have with probability $1-\delta_{\Phi}$:
\begin{eqnarray*}
\| \mb \Phi \| \leq \sqrt{D/d} + \sqrt{(2 \log (2/\delta_\Phi))/d} + 1 \;\;\;\; \text{ and }\\ \;\;\;\; \| \mb \Phi^\dagger \| \leq \left[ \sqrt{D/d} - \sqrt{(2 \log (2/\delta_\Phi))/d} - 1 \right]^{-1}.
\end{eqnarray*}
Similarly, when $n > d$, and the observed features are well-distributed, we expect that $\| \mb X \mb \Phi \|$ is of order $\tilde O(\sqrt{n/d})$ and $\| (\mb X \mb \Phi)^{\dagger} \|$ is of order $\tilde O(\sqrt{d/n})$. Also note that the projections are norm-preserving and thus $m_{\max} \simeq 1$. We also have $\left\|\mb \Sigma^{-1}_\Phi\right\| \leq d$.
Assuming that $n \gg d$, we can rewrite the bound on the error up to logarithmic terms as:
\begin{align}
\tilde O\left( \sqrt{k \log D} \left( \| \mb w \|   \sqrt{\frac{1}{d}} + \delta_{\max} \sqrt{\frac{d}{n}} \right) \right) + \tilde O\left(\frac{d^3}{n}\right).
\end{align}
The $1/\sqrt d$ term is a part of the bias due to the projection (excess approximation error). The $\sqrt{d/n}$ term is the variance term that shrinks with larger training sets (estimation error). We clearly observe the trade-off with respect to the compressed dimension $d$. With the assumptions discussed above, we can see that projection of size $d = \tilde O(k \log D)$ should be enough to guarantee arbitrarily small bias, as long as $\| \mb w \|$ is small and $n = \tilde O(d^3)$ holds\footnote{Our crude analysis assumes that $n=\tilde O(d^3)$. We expect that this can be further brought down to $\tilde O(d^2)$ which is $\tilde O((k \log D)^2)$ by our choice of $d$.}.

The following two lemmas complete the proof on the shrinkage of the error in the value function prediction:

%Note: we have angular error with respect to the maximum, not the norm. i.e. |behat - be|_rho < epsilon*max(be). we want the lp form of it: |behat - be|_rho < epsilon * |be|_rho

\begin{lemma} \label{lemma:errcontract}
Let $V^\pi$ be the value function of a policy $\pi$ imposing stationary measure $\rho$, and let $e_V$ be the Bellman error under policy $\pi$ for an estimate $V$. Given a BEBF $\psi$ satisfying:
\begin{eqnarray}
\left\| \psi(\mb x) - e_V(\mb x) \right\|_{\rho(\mb x)} \leq \epsilon \left\| e_V(\mb x) \right\|_{\rho(\mb x)},
\end{eqnarray}
we have that: 
\begin{eqnarray}
\left\| V^\pi(\mb x) - (V(\mb x) + \psi(\mb x)) \right\|_{\rho(\mb x)} \leq (\gamma + \epsilon + \epsilon \gamma) \left\| V^\pi(\mb x) - V(\mb x)\right\|_{\rho(\mb x)}.
\end{eqnarray}
\end{lemma}
Theorem~\ref{theorem:cbebf} does not state the error in terms of $\left\| e_V(\mb x) \right\|_{\rho(\mb x)}$, but rather does it in term of the infinity norm $e_{\max}$. We expect a more careful analysis to give us a bound that could benefit directly from Lemma~\ref{lemma:errcontract}. However, we can still state the following immediate lemma about the contraction in the error:
\begin{lemma} \label{lemma:errcontractorsmall}
Let $V^\pi$ be the value function of a policy $\pi$ imposing stationary measure $\rho$, and let $e_V$ be the Bellman error under policy $\pi$ for an estimate $V$. Given a BEBF $\psi$ satisfying:
\begin{eqnarray}
\left\| \psi(\mb x) - e_V(\mb x) \right\|_{\rho(\mb x)} \leq c,
\end{eqnarray}
we have that after adding the BEBF to the estimated value, either the error contracts:
\begin{eqnarray}
\left\| V^\pi(\mb x) - (V(\mb x) + \psi(\mb x)) \right\|_{\rho(\mb x)} < \left\| V^\pi(\mb x) - V(\mb x)\right\|_{\rho(\mb x)},
\end{eqnarray}
or the error is already small:
\begin{eqnarray}
\left\| V^\pi(\mb x) - V(\mb x)\right\|_{\rho(\mb x)} \leq \frac{(1+\gamma)}{(1-\gamma)^2}c.
\end{eqnarray}
\end{lemma}
This means that if we can control the error in BEBFs by some small constant, we can shrink the error up to a factor of that constant.

\section{Empirical Analysis}
We evaluate our method on a challenging domain where the goal of the RL agent is to apply direct electrical neurostimulation such as to suppress epileptiform behavior in neural tissues. We use a generative model constructed from real-world data collected on slices of rat brain tissues~\citep{bush:2009}; the model is available in the RL-Glue framework. Observations are generated over a $5$-dimensional real-valued state space. The discrete action choice corresponds to selecting the frequency at which neurostimulation is applied. The model is observed at $5$ steps per second. The reward is $0$ for steps when a seizure is occurring at the time of stimulation, $1/41$ for when seizure happens without stimultion, $40/41$ for each stimulation pulse, and $1$ otherwise~\footnote{The choice of the reward model is motivated by medical considerations. See~\cite{bush:2009}.}.

One of the challenges of this domain is that it is difficult to know \textit{a priori} how to construct a good state representation.  We use tile-coding to convert the continuous variables into a high dimensional binary feature space. We encode the policy as a $6$th feature, divide each dimension into $6$ tiles and use $10$ randomly placed tile grids. That creates $10 \times 6^6 = 466,560$ features. Only $10$ of these are non-zero at any point, thus $k=10$.

We apply the best clinical fixed rate policy (stimulation is applied at a consistent 1Hz) to collect our sample set~\citep{bush:2009}. Since the true value function is not known for this domain, we thus define our error in terms of Monte Carlo returns on a separate test set. Give a test set of size $l$, Monte Carlo returns are defined to be the discounted sum of rewards observed at each point, denoted by $U(\mb x_i)$. Now for any estimated value function $V$, we define the \emph{return prediction error} (RP error) to be $\sqrt{ \frac{1}{l} \sum^l_{i=1} \left(U(\mb x_i) - V(\mb x_i)\right)^2}$.

In our first experiment, we analyze the RP error as a function of the number of generated BEBFs, for different selections of the size of projection $d$. We run these experiments with two sample sizes: $500$ and $1500$. The projection sizes are either $10$, $20$ or $30$. Fixing $d$, we apply many iterations of the algorithm and observe the RP error on a testing set of size $l=5000$. To account for the randomness induced by the projections, we run these experiments 10 times each, and take the average. Figure~\ref{exp1fig} includes the results under the described setting.

It can seen in both plots in Figure~\ref{exp1fig}, that the RP error decreases to some minimum value after a number of BEBFs are generated, and then the error start increasing slightly when more BEBFs are added to the estimate. The increase is due to over-fitting and can be easily avoided by cross-validation. As stated before, this work does not include any analysis on the optimal number of iterations. Discussion on the possible methods for such optimization is an interesting avenue of future work.

\begin{figure}[ht]
\begin{center}
\includegraphics[scale = 0.5]{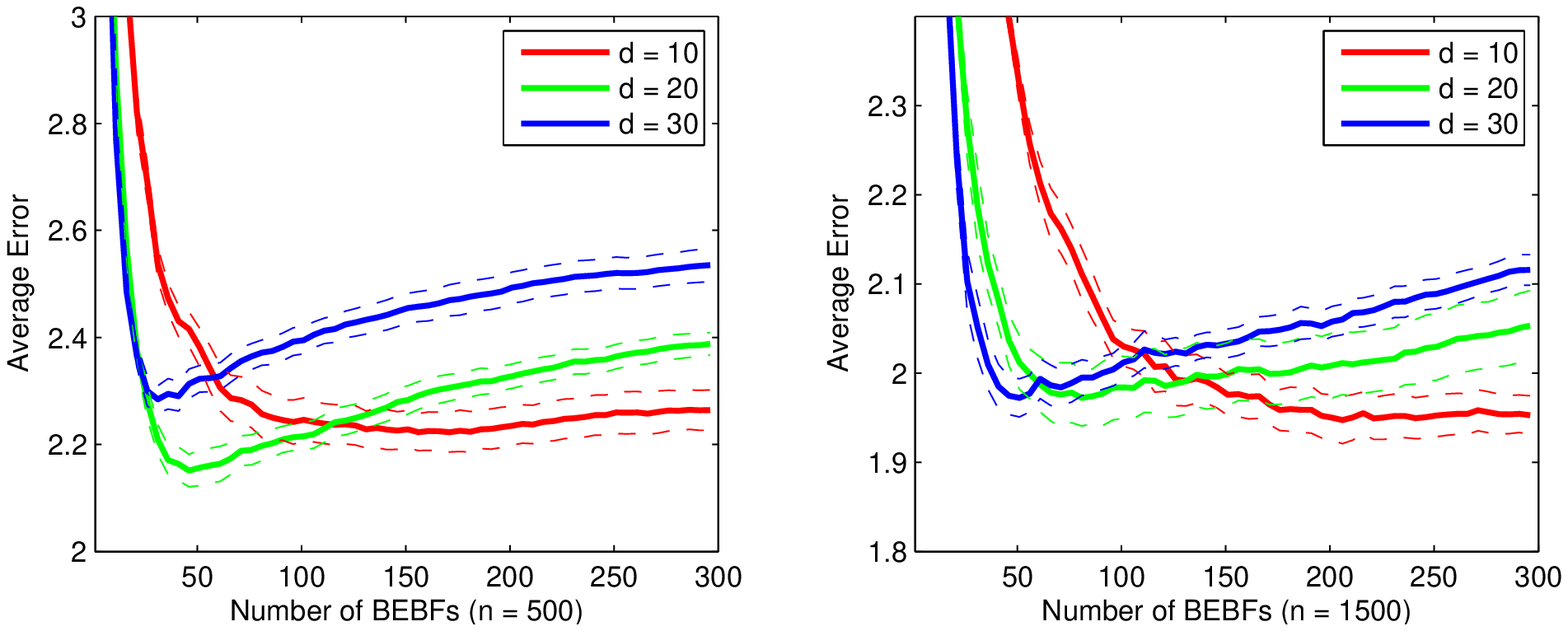}
\end{center}
\vspace{-10px}
\caption{\label{exp1fig} RP error of CBEBF for different number of projections, under different choices of $d$,}

\vspace{-5px}\center{averaged over 10 trials. The dashed lines indicate $\pm 1$-STD of the mean.}
\end{figure}

As expected, the optimal number of BEBFs depend heavily on the size of the projection: the smaller the projection, the more BEBFs need to be added. It is interesting to note that even though the minimum happens at different places, the value of the minimum RP error is not varying much as a function of the projection size. The difference gets even smaller with larger sample sizes. This means that the method is relatively robust with respect to the choice of $d$. We also observed small variance in the value of the optimal RP error, further confirming the robustness of the algorithm on this domain.

There are only a few methods that can be compared against our algorithm due to the high dimensional feature space. Direct regression on the original space with LSTD type algorithms (regularized or otherwise) is impossible due to the computational complexity\footnote{Analysis of sparse linear solvers, such as LSQR~\cite{paige1982lsqr}, is an interesting future work.}. We expect most feature \emph{selection} methods to perform poorly here, since all the features are of small and equal importance (note the different type of sparsity we assume in our work). The two main alternatives are randomized feature extraction (e.g. LSTD with random projections~\cite{ghavamzadeh2010lstd}) and online stochastic gradient type methods (e.g. GQ ($\lambda$) algorithm~\cite{maei2010gq}).

LSTD with random projections (Compressed LSTD, CLSTD), discussed in \cite{ghavamzadeh2010lstd}, is a simple algorithm in which one applies random projections to reduce the dimension of the state space to a manageable size, and then applies LSTD on the compressed space. We compare the RP error of CLSTD against our method. Among the gradient type methods, we chose the GQ ($\lambda$) algorithm~\cite{maei2010gq}, as it was expected to provide good consistency. However, since the algorithm was very sensitive to the choice of the learning rate schedule, the initial guess of the weight vector and the $\lambda$ parameter, we failed to tune it to outperform even the CLSTD. The results on the GQ ($\lambda$) algorithm are thus excluded from this section and should be addressed in future works~\footnote{A fair comparison cannot be made with gradient type methods in the absence of a good learning rate schedule. Typical choices were not enough to provide decent results.}.

\begin{figure}[ht]
\begin{center}
\includegraphics[scale = 0.6]{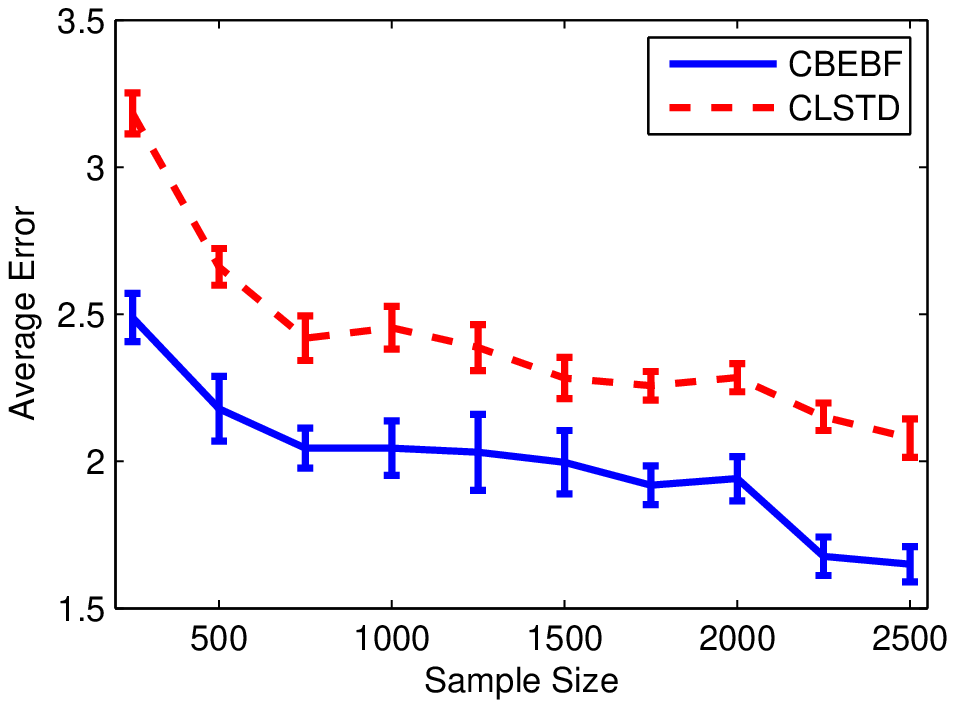}
\end{center}
\vspace{-10px}
\caption{\label{exp2fig} RP error of CBEBF vs. CLSTD for different sample sizes,}
\vspace{-5px}\center{averaged over 10 trials. The error bars are $\pm 1$-STD of the mean.}
\end{figure}

For a fair comparison between CBEBF and CLSTD, we assumed the existence of an oracle that would choose the optimal parameters for these method~\footnote{Note that since there are one or two parameters for these methods, cross-validation should be enough to choose the optimal parameter, though for simplicity the discussion of that is left out of this work.}. Therefore, we compare the best RP error on the testing set as we vary the parameters in question. Figure~\ref{exp2fig} shows the best RP error of the algorithms. For CLSTD, the best RP error is chosen among the solutions with varying projection sizes (extensive search). For CBEBF, we fix the projection size to be $20$, and vary the number of generated BEBFs (iteratively) to find the optimal number of iterations that minimizes the RP error.

As seen in Figure~\ref{exp2fig}, our method consistently outperforms CLSTD with a large margin, and the results are more robust with smaller variance. Comparing with the results presented in Figure~\ref{exp1fig}, even the over-fitted solutions of CBEBF seem to outperform the best results of CLSTD.

Each run of our algorithm with hundreds of BEBFs takes one or two minutes when working with thousands of samples and half a million features. The algorithm can easily scale to run with larger sample sizes and higher dimensional spaces, though a comparison cannot be made with CLSTD, since CLSTD (with optimal sizes of projection) fails to scale with increasing number of samples and dimensions.

\section{Discussion}
In this work, we provide a simple, fast and robust feature extraction algorithm for policy evaluation in sparse and high dimensional state spaces. Using recent results on the properties of random projections, we prove that in sparse spaces, random projections of sizes logarithmic in the original dimension are enough to preserve the linearity. Therefore, BEBFs can be generated on compressed spaces induced by small random projections. Our finite sample analysis provides guarantees on the reduction of error after the addition of the discussed BEBFs.

Empirical analysis on a high dimensional space with unknown value function structure shows that CBEBF vastly outperforms LSTD with random projections and easily scales to larger problems. It is also more consistent in the output and has a much smaller memory complexity. We expect this behaviour to happen under most common state spaces. However, more empirical analysis should be done to confirm such hypothesis. Since the focus of this work is on feature extraction with minimal domain knowledge using agnostic random projections, we avoided the commonly used problem domains with known structures in the value function (e.g. mountain car~\cite{sutton98}).

Compared to other regularization approaches to RL~\cite{kolter2009regularization,farahmand2010nips,parr2010nips}, our random projection method does not require complex optimization, and thus is faster and more scalable.

Of course finding the optimal choice of the projection size and the number of iterations is an interesting subject of future research.  We expect the use of cross-validation to suffice for the selection of the optimal parameters due to the robustness in the choice of values. A tighter theoretical bound might also help provide an analytical closed form answer to these questions.

Our assumption of the linearity of the Bellman error in the original space might be too strong for some state spaces. We avoided non-linearity in the original space to simplify the analysis. However, most of the discussions can be rephrased in terms of the \emph{projected Bellman error} to provide more general results (e.g. see~\cite{parr2007analyzing}).

\newpage

\begin{huge}\begin{center}\textbf{Appendix}\end{center}\end{huge}\vspace{1em}

We start with concentration bound on the rapidly mixing Markov processes. These will be used to bound the variance of approximations build upon the observed values.

\section{Concentration Bounds for Mixing Chains}
We give an extension of Bernstein's inequality based on~\cite{samson2000concentration}.

Let $\mb x_1, \ldots, \mb x_n$ be a time-homogeneous Markov chain with transition kernel $T(\cdot|\cdot)$ taking values in some measurable space $\mathcal X$.
We shall consider the concentration of the average of the Hidden-Markov Process 
\[(\mb x_1,f(\mb x_1)),\ldots,(\mb x_n,f(\mb x_n)),\]
 where $f:\mathcal X \rightarrow [0,b]$ is a fixed measurable function. 
To arrive at such an inequality, we need a characterization of how fast $(\mb x_i)$ forgets its past.

For $i>0$, let $T^i(\cdot|x)$ be the $i$-step transition probability kernel: $T^i(A|\mb x) = \Pr \{\mb x_{i+1}\in A\,|\,\mb x_1=\mb x\}$ (for all $A\subset \mathcal X$ measurable).
Define the upper-triangular matrix $\mb \Gamma_n=(\gamma_{ij})\in \real^{n\times n}$  as follows:

\begin{eqnarray}
\label{eq:gammadef}
	\gamma_{ij}^2 = \sup_{(\mb x, \mb y) \in \mathcal X^2} 
	        \| T^{j-i}(\cdot|\mb x) - T^{j-i}(\cdot|\mb y) \|_\text{TV}.
\end{eqnarray}

for $1\leq i < j \leq n$ and let $\gamma_{ii}=1$ ($1\le i \le n$).

\newcommand{\Id}{\mathbf{I}}
\newcommand{\norm}[1]{\left\Vert#1\right\Vert}
\newcommand{\Prob}[1]{{\mathbb P}\left(#1\right)}    % Probabilities; example: \Prob{X>\eps}<1-\delta
\newcommand{\EE}[1]{{\mathbb E}\left[#1\right]}      % Expectations

Matrix $\mb \Gamma_n$, and its operator norm $\norm{\mb \Gamma_n}$ w.r.t. the Euclidean distance, are the measures of dependence for the random sequence $\mb x_1, \mb x_2, \ldots, \mb x_n$. For example if the $\mb x_i$'s are independent, $\mb \Gamma_n = \Id$ and $\norm{\mb \Gamma_n} = 1$.
In general $\norm{\mb \Gamma_n}$, which appears in the forthcoming concentration inequalities for dependent sequences, can grow with $n$. Since the concentration bounds are homogeneous in $n/\norm{\mb \Gamma_n}^2$, a larger value $\norm{\mb \Gamma_n}^2$ means a smaller ``effective'' sample size.

We say that a time-homogeneous Markov chain \emph{uniformly quickly forgets its past} if
\begin{eqnarray}
\tau=\sup_{n\ge 1}\norm{\mb \Gamma_n}^2<+\infty.
\end{eqnarray}
Further, $\tau$ is called the \emph{forgetting time} of the chain. Conditions under which a  Markov chain uniformly quickly forgets its past are of major interest. For further discussion on this, see~\cite{farahmand2011model}.

The following result from~\cite{farahmand2011model} is a trivial corollary of Theorem~2 of \citet{samson2000concentration} (Theorem~2 is stated for empirical processes and can be considered as a generalization of Talagrand's inequality to dependent random variables):

\begin{theorem}\label{thm:SamsonTheorem2}
Let $f$ be a measurable function on $\mathcal X$ whose values lie in $[0,b]$,
 $(\mb x_i)_{1\le i\le n}$ be a homogeneous Markov chain taking values in $\mathcal X$
 and let $\mb \Gamma_n$ be the matrix with elements defined by~\eqref{eq:gammadef}.
Let
\begin{equation*}
	z =  \frac1n \sum_{i=1}^n f(\mb x_i).
\end{equation*}	
Then, for every $\epsilon \geq 0$,
\begin{align*}
	\Prob{ z - \EE{z} \geq \epsilon } & \leq \exp \left(	-\frac{\epsilon^2\, n}{2b \norm{\mb \Gamma_n}^2 (\EE{z} + \epsilon)}		\right), \\
	\Prob{ \EE{z} - z \geq \epsilon } & \leq \exp \left(	-\frac{\epsilon^2\, n}{2b \norm{\mb \Gamma_n}^2 \EE{z}}		\right).
\end{align*}
\end{theorem}

The following is an immediate application of the above theorem:

\begin{lemma}\label{lemma:mixing}
Let $f$ be a measurable function over $\mathcal X$ whose values lie in $[0,b]$. Let $\hat f$ be the empirical average of $f$ over the sample collected on the Markov chain. Under proper mixing conditions for the sample, there exists constants $c_1>0$, $c_2\ge 1$ which depend only on $T$ such that for any $0<\xi<1$, w.p. $1-\xi$:
\begin{eqnarray}
\left| \EE{f} - \hat f \right|  \le b \sqrt{ \frac{c_1}{n} \, \log\left(\frac{c_2}{\xi} \right) }\,.
\end{eqnarray}
\end{lemma}

\section{Proof of The Theorem 2}
\begin{proof}[Proof of Theorem 2]

To begin the proof of the main theorem, first note that we can write the TD-errors as the sum of Bellman errors and some noise term: $\delta_t = e_V(\mb x_t) + \eta_t$. These noise terms form a series of martingale differences, as their expectation is 0 given all the history up to that point:
\begin{eqnarray}
\EE{\eta_t | \mb x_1 \dots \mb x_t, r_1 \dots r_{t-1}} = 0.
\end{eqnarray}

We also have that the Bellman error is linear in the features, thus in vector form:
\begin{eqnarray}
\mb \delta = \mb X \mb w + \eta.
\end{eqnarray}
Using random projections, in the compressed space we have:
\begin{eqnarray}
\mb \delta = (\mb X \mb \Phi) (\mb \Phi^T \mb w) + \mb b + \eta,
\end{eqnarray}
where $\mb b$ is the vector of bias due to the projection. Let $b_{\max} = \epsilon^{(\xi_1)}_{\text{prj}} \|\mb w\|$. We have from Lemma 1 that with probability $1-\xi_1$, for all $\mb x \in \mathcal X$:
\begin{eqnarray}
\left|(\mb x^T \mb \Phi) (\mb \Phi^T \mb w) - e_V(\mb x)\right| =
\left|(\mb x^T \mb \Phi) (\mb \Phi^T \mb w) - \mb x^T \mb w \right| \leq b_{\max}.
\end{eqnarray}
Thus, $\mb b$ is element-wise bounded in absolute value by $b_{\max}$ with high probability.
The weighted $L^2$ error in regression to the TD-error as compared to the Bellman error will thus be:
\begin{eqnarray}
\left\| \mb x^T \mb \Phi \mb w^{(\Phi)}_{\text{ols}} - e_V(\mb x) \right\|_{\rho(\mb x)} \hspace{-10px}
 &=& \hspace{-5px} \left\| (\mb x^T \mb \Phi) (\mb X \mb \Phi)^\dagger [(\mb X \mb \Phi) (\mb \Phi^T \mb w) + \mb b + \eta] - e_V(\mb x) \right\|_{\rho(\mb x)} \nonumber \\
 &=& \hspace{-5px} \left\|  (\mb x^T \mb \Phi) (\mb \Phi^T \mb w) + (\mb x^T \mb \Phi) (\mb X \mb \Phi)^\dagger \mb b + (\mb x^T \mb \Phi) (\mb X \mb \Phi)^\dagger \eta - e_V(\mb x) \right\|_{\rho(\mb x)} \nonumber \\
 &\leq& \hspace{-5px} \left\|  (\mb x^T \mb \Phi) (\mb \Phi^T \mb w) - e_V(\mb x) \right\|_{\rho(\mb x)} \nonumber \\
 && \hspace{-5px} + \left\|  (\mb x^T \mb \Phi) (\mb X \mb \Phi)^\dagger \mb b \right\|_{\rho(\mb x)} 
 + \left\|  (\mb x^T \mb \Phi) (\mb X \mb \Phi)^\dagger \eta \right\|_{\rho(\mb x)} \nonumber \\
 &\leq& \hspace{-5px} b_{\max}
 + \left\|  (\mb x^T \mb \Phi) (\mb X \mb \Phi)^\dagger \mb b \right\|_{\rho(\mb x)}
 + \left\|  (\mb x^T \mb \Phi) (\mb X \mb \Phi)^\dagger \eta \right\|_{\rho(\mb x)}.
\end{eqnarray}
The second term is the regression to the bias, and the third term is the regression to the noise. We present lemmas that bound these terms. The theorem is proved by the application of Lemma~\ref{lemma:bias} and Lemma~\ref{lemma:var}.
\end{proof}

\subsection{Bounding the Regression to Bias Terms}

\begin{lemma} \label{lemma:bias}
Under the conditions and with probability defined in Theorem 2:
\begin{eqnarray}
\left\|  (\mb x^T \mb \Phi) (\mb X \mb \Phi)^\dagger \mb b \right\|_{\rho(\mb x)}
\leq b_{\max} \left( 2 + 2 m_{\max} \left\|(\mb X \mb \Phi)^\dagger\right\|  \left( \sqrt \frac{n}{d} + \sqrt[4]{ c_3 n d \, \log\frac{c_4 \sqrt d}{\xi} } \right) \right).
\end{eqnarray}
\end{lemma}

\begin{proof}
Define $\mb w_{\mb X} = (\mb X \mb \Phi)^\dagger \mb b$. Also define $\|.\|_n$ to be the weighted $L^2$ norm uniform on the sample set $X$:
\begin{eqnarray}
\|f(\mb x)\|^2_n = \frac{1}{n} \sum^n_{i=1} (f(\mb X_i))^2.
\end{eqnarray}

We start by bounding the empirical norm $\|(\mb x^T \mb \Phi) \mb w_{\mb X}\|_n$. Given that $(\mb X \mb \Phi) \mb w_{\mb X}$ is the OLS regression on the observed points, its sum of squared errors should not be greater than any other linear regression, including the vector $0$, thus $\|(\mb x^T \mb \Phi) \mb w_{\mb X} - b(\mb x)\|_n \leq \|b(\mb x)\|_n$. We get:
\begin{eqnarray}
\|(\mb x^T \mb \Phi) \mb w_{\mb X}\|_n \leq \|(\mb x^T \mb \Phi) \mb w_{\mb X} - b(\mb x)\|_n + \|b(\mb x)\|_n \leq 2 \|b(\mb x)\|_n \leq 2 b_{\max}.
\label{emp-bound}
\end{eqnarray}

Let $\mathcal W = \{\mb u \in \md R^d \text{ s.t. } \|\mb u\| \leq 1\}$. Let $S \subset \mathcal W$ be an $\epsilon$-grid cover of $\mathcal W$:
\begin{eqnarray}
\forall \mb v \in \mathcal W \; \exists \mb u \in S: \|\mb u - \mb v\| \leq \epsilon. \label{ep-net}
\end{eqnarray}
It is easy to prove (see e.g. Chapter 13 of \cite{lorentz1996constructive}) that these conditions can be satisfied by choosing a grid of size $|S| \leq (3/\epsilon)^d$ ($S$ fills up the space within $\epsilon$ distance). Applying union bound to Lemma~\ref{lemma:mixing} (let $f(\mb x) = ((\mb x^T \mb \Phi) \mb u)^2$) for all elements in $S$, we get with probability no less than $1-\xi$:
\begin{eqnarray}
\forall \mb u \in S : \|(\mb x^T \mb \Phi) \mb u \|^2_\rho(\mb x) \leq \|(\mb x^T \mb \Phi) \mb u \|^2_n + m^2_{\max} \sqrt{ \frac{c_1}{n} \, \log \frac{c_2 |S|}{\xi} } . \label{net-bound}
\end{eqnarray}

Let $\mb w'_{\mb X} = \mb w_{\mb X}/ \| \mb w_{\mb X} \|$. For any $\mb X$, since $\mb w'_{\mb X} \in \mathcal W$, there exists $\mb w'' \in S$ such that $ \| \mb w'_{\mb X} - \mb w'' \| \leq \epsilon$. Therefore, under event \ref{net-bound} we have:
\begin{eqnarray}
 && \hspace{-3.5em} \left\|  (\mb x^T \mb \Phi) \mb w_{\mb X} \right\|_{\rho(\mb x)} = \left\| \mb w_{\mb X} \right\|   \left\|  (\mb x^T \mb \Phi) \mb w'_{\mb X} \right\|_{\rho(\mb x)}\\
 &\leq& \hspace{-0.7em} \left\| \mb w_{\mb X} \right\|   \left( \left\|  (\mb x^T \mb \Phi) (\mb w'_{\mb X} - \mb w'') \right\|_{\rho(\mb x)}  + \left\|  (\mb x^T \mb \Phi) \mb w'' \right\|_{\rho(\mb x)} \right)\\
 &\leq& \hspace{-0.7em} \left\| \mb w_{\mb X} \right\|   \left( m_{\max} \left\|\mb w'_{\mb X} - \mb w''\right\|  + \left\|  (\mb x^T \mb \Phi) \mb w'' \right\|_n + m_{\max} \sqrt[4]{ \frac{c_1}{n}  \log\frac{c_2 |S|}{\xi} } \right) \label{use-net-bound}\\
 &\leq& \hspace{-0.7em} \left\| \mb w_{\mb X} \right\|   \left( m_{\max} \epsilon  + \left\|  (\mb x^T \mb \Phi) (\mb w'' - \mb w'_{\mb X}) \right\|_n + \left\|  (\mb x^T \mb \Phi) \mb w'_{\mb X} \right\|_n + m_{\max} \sqrt[4]{ \frac{c_1}{n} \log\frac{c_2 |S|}{\xi} } \right) \label{use-ep-net-1}\\
 &\leq& \hspace{-0.7em} \left\| \mb w_{\mb X} \right\|   \left( m_{\max} \epsilon  + m_{\max} \epsilon + \left\|  (\mb x^T \mb \Phi) \mb w'_{\mb X} \right\|_n + m_{\max} \sqrt[4]{ \frac{c_1}{n}  \log\frac{c_2 |S|}{\xi} } \right) \label{use-ep-net-2}\\
 &\leq& \hspace{-0.7em} \left\|  (\mb x^T \mb \Phi) \mb w_{\mb X} \right\|_n + m_{\max} \left\| \mb w_{\mb X} \right\|  \left( 2 \epsilon + \sqrt[4]{ \frac{c_1}{n}  \log\frac{c_2 |S|}{\xi} } \right).
\end{eqnarray}
Line~\ref{use-net-bound} uses Equation~\ref{net-bound}, and we use Equation~\ref{ep-net} in lines~\ref{use-ep-net-1} and \ref{use-ep-net-2}. Using the definition, we have that $\left\| \mb w_{\mb X} \right\| \leq \left\|(\mb X \mb \Phi)^\dagger\right\| b_{\max} \sqrt n$. Thus, using Equation~\ref{emp-bound} we get:
\begin{eqnarray}
\left\|  (\mb x^T \mb \Phi) \mb w_{\mb X} \right\|_{\rho(\mb x)} &\leq& 2 b_{\max} +  m_{\max} \left\|(\mb X \mb \Phi)^\dagger\right\| b_{\max} \sqrt n \left( 2 \epsilon + \sqrt[4]{ \frac{c_1}{n} \, \log\frac{c_2 |S|}{\xi} } \right).
\end{eqnarray}
Setting $\epsilon = 1/\sqrt d$ and substituting $|S|$ we get:
\begin{eqnarray}
\left\|  (\mb x^T \mb \Phi) \mb w_{\mb X} \right\|_{\rho(\mb x)} &\leq& b_{\max} \left( 2 + m_{\max} \left\|(\mb X \mb \Phi)^\dagger\right\|  \left( 2 \sqrt \frac{n}{d} + \sqrt[4]{ c_1 n \, \log\frac{c_2 (3/\epsilon)^d}{\xi} } \right) \right),
\end{eqnarray}
which proves the lemma after simplification.
\end{proof}

\subsection{Bounding the Regression to Noise Terms}
To bound the regression to the noise, we need the following lemma on martingales:
\begin{lemma}\label{lemma:mart}
Let $\mb M$ be a matrix of size $l \times n$, in which column $t$ is a function of $\mb x_t$. Then with probability $1-\xi$ we have:
\begin{eqnarray}
\| \mb M \mb \eta \| \leq \delta_{\max} \| \mb M \|_F \sqrt{2 \log \frac{2l}{\xi}}.
\end{eqnarray}
\end{lemma}

\begin{proof}
The inner product between each row of $\mb M$ and $\mb \eta$ can be bounded by a concentration inequality on martingales each failing with probability less than $\xi/l$. The lemma follows immediately by adding up the inner products.
\end{proof}

The following lemma based on mixing conditions is also needed to bound the variance term.

\begin{lemma}\label{lemma:sigma}
With the conditions of the theorem, with probability $1-\xi_4$, there exists a $\mb Y^{d\times d}$ with all the elements in $[-1,1]$, and thus $\|\mb Y\| \leq d$, such that:
\begin{eqnarray}
\frac{1}{n}  (\mb X \mb \Phi)^T \mb X \mb \Phi =  \mb \Sigma_\Phi + \epsilon_0 \mb Y,
\end{eqnarray}
where $\epsilon_0 = m^2_{\max} \sqrt{\frac{c_3}{n} \log \frac{c_4 d^2}{\xi_4}}$.
\end{lemma}
Stated otherwise, if $\mb Y = \frac{1}{\epsilon_0} \left(\frac{1}{n} (\mb X \mb \Phi)^T \mb X \mb \Phi - \mb \Sigma_\Phi \right)$, then with probability $1-\xi_4$, $\mb Y$ is element-wise bounded by $\pm 1$.
\begin{proof}
This is a simple application of Lemma~\ref{lemma:mixing} to all the elements in $\frac{1}{n} (\mb X \mb \Phi)^T \mb X \mb \Phi$ using union bound, as the expectation is $\mb \Sigma$, the chain is mixing and each element of $\mb X \mb \Phi$ is bounded by $m_{\max}$.
\end{proof}

With the above theorem, we can use the Taylor expansion of matrix inversion to have:
\begin{eqnarray}
((\mb X \mb \Phi)^T \mb X \mb \Phi)^{-1} =  \frac{1}{n}  (\mb \Sigma_\Phi + \epsilon_0 \mb Y)^{-1} = \frac{1}{n}  (\mb \Sigma^{-1}_\Phi - \epsilon_0 \mb \Sigma^{-1}_\Phi \mb Y \mb \Sigma^{-1}_\Phi + O(\epsilon^2_0)).
\end{eqnarray}

\begin{lemma} \label{lemma:var}
Under the conditions and probabiliy defined in Theorem 2:
\begin{eqnarray}
\left\|  (\mb x^T \mb \Phi) (\mb X \mb \Phi)^\dagger \eta \right\|_{\rho(\mb x)} &\leq& \frac{\delta_{\max} m_{\max}}{n} \left\|\mb  \Sigma^{-1}_\Phi\right\|   \| \mb X \mb \Phi \| \sqrt {2 k \log \frac{2D}{\xi_3}} \\
&& + \delta_{\max} m^3_{\max} \sqrt{\frac{d^3}{n^3}}  \left\|\mb \Sigma^{-1}_\Phi\right\|^2   \| \mb X \mb \Phi \| \sqrt {2 c_3 \log \frac{c_4 d^2}{\xi_4} \log \frac{2 d}{\xi_5}}\\
&& + \tilde O(n^{-2}).
\end{eqnarray}
\end{lemma}

\begin{proof}
Since $\mb X$ is of rank bigger than $d$, we have $d<n$, and with the use of random projections $\mb X \mb \Phi$ is full rank with probability 1 (see e.g. \cite{fard2012comp}). We can thus substitute the inverse by $[(\mb X \mb \Phi)^T \mb X \mb \Phi]^{-1} (\mb X \mb \Phi)^T$. Using Lemma~\ref{lemma:sigma}, we get with probability $1-\xi_4$ for all $\mb x \in \mathcal X$:
\begin{eqnarray}
&& \hspace{-20px}\left\|  (\mb x^T \mb \Phi) (\mb X \mb \Phi)^\dagger \eta \right\|_{\rho(\mb x)}  = \left\|  (\mb x^T \mb \Phi) ((\mb X \mb \Phi)^T \mb X \mb \Phi)^{-1} (\mb X \mb \Phi)^T \eta \right\|_{\rho(\mb x)}\\
&=& \left\|  (\mb x^T \mb \Phi) \left[\frac{1}{n}  (\mb \Sigma^{-1}_\Phi - \epsilon_0 \mb \Sigma^{-1}_\Phi \mb Y \mb \Sigma^{-1}_\Phi + O(\epsilon^2_0))\right] (\mb X \mb \Phi)^T \eta \right\|_{\rho(\mb x)}\\
&\leq& \left\|  \frac{1}{n} (\mb x^T \mb \Phi) \mb \Sigma^{-1}_\Phi (\mb X \mb \Phi)^T \eta \right\|_{\rho(\mb x)} 
+ \left\|  \frac{\epsilon_0}{n}  (\mb x^T \mb \Phi) \mb \Sigma^{-1}_\Phi \mb Y \mb \Sigma^{-1}_\Phi (\mb X \mb \Phi)^T \eta \right\|_{\rho(\mb x)}
+ O\left(\frac{\epsilon^2_0}{n}\right). \label{label:varterms}
\end{eqnarray}
To bound the first term, let $e_i$ be the $i$th column of $\Psi$ (see definition of $\mathcal X$ in the notation section). Thus $\{\mb e_i\}_{1 \leq i \leq D}$ is an orthonormal basis under which $\mb x \in \mathcal X$ is sparse, and all $\mb e_i$'s are in $\mathcal X$. Applying Lemma~\ref{lemma:mart}, $D$ times, we get that for all $\mb e_i$, with probability $1-\xi_3/D$:
\begin{eqnarray}
\left|  \frac{1}{n} (\mb e_i^T \mb \Phi) \mb \Sigma^{-1}_\Phi (\mb X \mb \Phi)^T \eta \right| &\leq& \delta_{\max} \left\|  \frac{1}{n} (\mb e_i^T \mb \Phi) \mb \Sigma^{-1}_\Phi (\mb X \mb \Phi)^T \right\| \sqrt{2 \log \frac{2D}{\xi_3}} \\
&\leq& \frac{\delta_{\max} m_{\max}}{n}   \left\|\mb  \Sigma^{-1}_\Phi\right\|   \| \mb X \mb \Phi \| \sqrt {2 \log \frac{2D}{\xi_3}}. \label{label:basisbound}
\end{eqnarray}
The union bound gives us that Line~\ref{label:basisbound} hold simultaneously for all $\mb e_i$'s with probability $1-\xi_3$. Therefore with probability $1-\xi_3$ for any $\mb x = \sum_i \alpha_i \mb e_i$:
\begin{eqnarray}
\left( \frac{1}{n} (\mb x^T \mb \Phi) \mb \Sigma^{-1}_\Phi (\mb X \mb \Phi)^T \eta \right)^2 &=& \left(\frac{1}{n} \sum^D_{i=1} (\alpha_i \mb e_i^T \mb \Phi) \mb \Sigma^{-1}_\Phi (\mb X \mb \Phi)^T \eta \right)^2\\
&\leq& \left( \frac{1}{n} \sum^D_{i=1} |\alpha_i| \left|  (\mb e_i^T \mb \Phi) \mb \Sigma^{-1}_\Phi (\mb X \mb \Phi)^T \eta \right| \right)^2\\
&\leq& \left( \frac{\delta_{\max} m_{\max}}{n}   \left\|\mb  \Sigma^{-1}_\Phi\right\|   \| \mb X \mb \Phi \| \sqrt {2 \log \frac{2D}{\xi_3}}  \right)^2 \left( \sum^D_{i=1} |\alpha_i| \right)^2.
\end{eqnarray}
Because $\mb x$ is $k$-sparse, we have that $\sum^D_{i=1} |\alpha_i| \leq \sqrt k \|\mb x\| \leq \sqrt k$. As the above holds for all $\mb x = \mathcal X$, it holds for the expectation under $\rho$. We thus get:
\begin{eqnarray}
\left\| \frac{1}{n} (\mb x^T \mb \Phi) \mb \Sigma^{-1}_\Phi (\mb X \mb \Phi)^T \eta \right\|_{\rho(\mb x)} &\leq& \frac{\delta_{\max} m_{\max}}{n}   \left\|\mb  \Sigma^{-1}_\Phi\right\|   \| \mb X \mb \Phi \| \sqrt {2k \log \frac{2D}{\xi_3}}. \label{label:firsttermvar}
\end{eqnarray}

For the second term of Line~\ref{label:varterms}, we first split and then apply Lemma~\ref{lemma:mart}. 
\begin{eqnarray}
\hspace{-10px} \left|  \frac{\epsilon_0}{n}  (\mb x^T \mb \Phi) \mb \Sigma^{-1}_\Phi \mb Y \mb \Sigma^{-1}_\Phi (\mb X \mb \Phi)^T \eta \right|
&\leq& \frac{\epsilon_0}{n} \left\|\mb \Phi^T \mb x\right\| \left\|\mb \Sigma^{-1}_\Phi\right\|^2 \|\mb Y\| \left\| (\mb X \mb \Phi)^T \eta \right\|.
\end{eqnarray}
Using Lemma~\ref{lemma:sigma}, we have with probability $1-\xi_4$ that $\| \mb Y \| \leq d$. Applying Lemma~\ref{lemma:mart} to the $\left\| (\mb X \mb \Phi)^T \eta \right\|$ term we get with probability $1-\xi_5$ for all $\mb x \in \mathcal X$:
\begin{eqnarray}
\hspace{-10px} \left|  \frac{\epsilon_0}{n}  (\mb x^T \mb \Phi) \mb \Sigma^{-1}_\Phi \mb Y \mb \Sigma^{-1}_\Phi (\mb X \mb \Phi)^T \eta \right|
&\leq& \frac{\epsilon_0}{n} \left\|\mb \Phi^T \mb x\right\| \left\|\mb \Sigma^{-1}_\Phi\right\|^2 d \delta_{\max} \left\|\mb X \mb \Phi\right\|_F \sqrt {2 \log \frac{2d}{\xi_5}}\\
&\leq& \frac{\epsilon_0}{n} m_{\max} \left\|\mb \Sigma^{-1}_\Phi\right\|^2 d \delta_{\max} \sqrt d \left\|\mb X \mb \Phi\right\| \sqrt {2 \log \frac{2d}{\xi_5}}.
\end{eqnarray}
As the above holds for all $\mb x \in \mathcal X$, it holds for any expectation on with measures defined on $\mathcal X$:
\begin{eqnarray}
\left\|  \frac{\epsilon_0}{n}  (\mb x^T \mb \Phi) \mb \Sigma^{-1}_\Phi \mb Y \mb \Sigma^{-1}_\Phi (\mb X \mb \Phi)^T \eta \right\|_{\rho(\mb x)}
\leq \frac{\epsilon_0 \delta_{\max} m_{\max}\sqrt{d^3}}{n}  \left\|\mb \Sigma^{-1}_\Phi\right\|^2 \left\|\mb X \mb \Phi\right\| \sqrt {2 \log \frac{2d}{\xi_5}}. \label{label:secondtermvar}
\end{eqnarray}
Substituting $\epsilon_0$ of Lemma~\ref{lemma:sigma}, and using Lines~\ref{label:firsttermvar} and \ref{label:secondtermvar} into \ref{label:varterms} will finish the proof.
\end{proof}

\section{Proof of Error Contraction Lemmas}
This section will finish the proof of the lemmas presented in the paper.

\subsection{Proof of Lemma 3}
\begin{proof}[Proof of Lemma 3]
We have that $V^\pi$ is the fixed point to the Bellman operator (i.e. $\mathcal T V^\pi = V^\pi$), and that the operator is a contraction with respect to the weighted $L^2$ norm on the stationary distribution $\rho$~\cite{van1998learning}:
\begin{eqnarray}
\left\| \mathcal T V(\mb x) - \mathcal T V'(\mb x) \right\|_{\rho(\mb x)} \leq \gamma \left\| V(\mb x) - V'(\mb x) \right\|_{\rho(\mb x)}.
\end{eqnarray}
We thus have:
\begin{eqnarray}
&& \hspace{-40px}\left\| V^\pi(\mb x) - (V(\mb x) + \psi(\mb x)) \right\|_{\rho(\mb x)}\\
&\leq& \left\| V^\pi(\mb x) - \mathcal T V(\mb x) \right\|_{\rho(\mb x)} + \left\|  (\mathcal T V(\mb x) - V(\mb x)) - \psi(\mb x) \right\|_{\rho(\mb x)}\\
&\leq& \left\| \mathcal T V^\pi(\mb x) - \mathcal T V(\mb x) \right\|_{\rho(\mb x)} + \epsilon \left\|  \mathcal T V(\mb x) - V(\mb x) \right\|_{\rho(\mb x)}\\
&\leq& \gamma \left\|  V^\pi(\mb x) -  V(\mb x) \right\|_{\rho(\mb x)} + \epsilon \left( \left\|  \mathcal T V(\mb x) - \mathcal T V^\pi(\mb x) \right\|_{\rho(\mb x)} + \left\| V^\pi(\mb x) - V(\mb x) \right\|_{\rho(\mb x)}  \right)\\ 
&\leq& (\gamma + \epsilon \gamma + \epsilon) \left\| V^\pi(\mb x) - V(\mb x)\right\|_{\rho(\mb x)}.
\end{eqnarray}
\end{proof}

\subsection{Proof of Lemma 4}
\begin{proof}[Proof of Lemma 4]
We have that:
\begin{eqnarray}
\left\| V^\pi(\mb x) - V(\mb x)\right\|_{\rho(\mb x)} &\leq& \left\| \mathcal T V^\pi(\mb x) - \mathcal T V(\mb x) \right\|_{\rho(\mb x)} + \left\|  \mathcal T V(\mb x) - V(\mb x) \right\|_{\rho(\mb x)}\\
&\leq& \gamma \left\| V^\pi(\mb x) - V(\mb x) \right\|_{\rho(\mb x)} + \left\|  \mathcal T V(\mb x) - V(\mb x) \right\|_{\rho(\mb x)},
\end{eqnarray}
and thus:
\begin{eqnarray}
\left\| V^\pi(\mb x) - V(\mb x)\right\|_{\rho(\mb x)} &\leq& \frac{1}{1-\gamma} \left\|  \mathcal T V(\mb x) - V(\mb x) \right\|_{\rho(\mb x)}.
\end{eqnarray}
Let $\epsilon = c / \left\| e_V(\mb x) \right\|_{\rho(\mb x)}$. If the contraction does not happen, then due to Lemma 3, we must have:
\begin{eqnarray}
\gamma + \epsilon \gamma + \epsilon \geq 1 \;\; &\Rightarrow& \;\; \epsilon > \frac{1-\gamma}{1+\gamma} \\
\;\; &\Rightarrow& \;\;  \left\|  \mathcal T V(\mb x) - V(\mb x) \right\|_{\rho(\mb x)} \leq \frac{1+\gamma}{1-\gamma} c\\
\;\; &\Rightarrow& \;\;  \left\| V^\pi(\mb x) - V(\mb x)\right\|_{\rho(\mb x)} \leq \frac{1+\gamma}{(1-\gamma)^2} c.
\end{eqnarray}

\end{proof}

\newpage
\bibliographystyle{unsrtnat}
\small{
\bibliography{cbebf-arxiv}
}

\end{document}